\documentclass{article}
\usepackage{ijcai26}

\usepackage{times}
\usepackage{soul}
\usepackage{url}
\usepackage[hidelinks]{hyperref}
\usepackage[utf8]{inputenc}
\usepackage[small]{caption}
\usepackage{graphicx}
\usepackage{newfloat}
\usepackage{amsmath,amssymb,amsthm}
\usepackage[vlined,linesnumbered,ruled,titlenotnumbered]
{algorithm2e}
\usepackage{booktabs}
\usepackage[switch]{lineno}
\usepackage{xcolor}
\usepackage{thm-restate}

\newcommand{\ones}[1]{|#1|_1}   
\newcommand{\zeros}[1]{|#1|_0}  

\newcommand{\JUMP}{m\text{-}\textsc{OJZJ}_k\xspace}

\newcommand{\JUMPFOUR}{4\text{-}\textsc{OJZJ}_k\xspace}

\newcommand{\RRRfull}{\textsc{RealRoyalRoad}\xspace}               

\newcommand{\OMMfull}{\textsc{OneMinMax}\xspace}        

\newcommand{\OJZJ}{\textsc{OJZJ}\xspace}                           
\newcommand{\OJZJfull}{\textsc{OneJumpZeroJump}\xspace}            

\newcommand{\expect}[1]{\mathrm{E}\left[#1\right]}        

\newcommand{\rp}{\mathrm{rp}}
\newcommand{\refer}{\mathcal{R}_p}

\newcommand{\nsga}{NSGA\nobreakdash-II\xspace}
\newcommand{\nsgaIII}{NSGA\nobreakdash-III\xspace}

\linenumbers

\newtheorem{theorem}     {Theorem}
\newtheorem{lemma}      [theorem]{Lemma}

\title{On the Impact of Crossover in Many-Objective Optimization: A Runtime Analysis of NSGA-III}

\author{
Andre Opris
    \affiliations
    Chair of Algorithms for Intelligent Systems, University of Passau, Passau, Germany
    \emails
    andre.opris@uni-passau.de
}

\begin{document}
\nolinenumbers
\maketitle

\begin{abstract}
In recent years, a theoretical understanding has rapidly advanced regarding how popular multi-objective evolutionary algorithms (MOEAs) can optimize many-objective problems. However, the benefits of using crossover in many-objective optimization are theoretically not understood, except for specifically designed benchmark functions tuned to particular crossover operators, and still lag significantly behind its practical use. In this paper, we build upon this line of research and present a theoretical runtime analysis of the widely used NSGA-III algorithm on the classical $m$-objective $m$-\OJZJfull function ($m$-\OJZJ for short). Our results demonstrate that NSGA-III with crossover optimizes $m$-\OJZJ asymptotically faster than NSGA-III without crossover for any number $m$ of objectives for huge parameter regimes. We complement our analysis by providing a lower runtime bound on $4$-OJZJ when crossover is turned off. 
\end{abstract}

\section{Introduction}
Multi-objective evolutionary algorithms (MOEAs), such as the non-dominated sorting genetic algorithm II (NSGA-II)~\cite{Deb2002} and its extension for many-objective optimization, \nsgaIII~\cite{DebJain2014}, have been used in thousands of applications and together have received more than 60,000 citations. Applications span a wide range of domains, including machine learning~\cite{ZHU20251028}, bioengineering~\cite{NSGAIIBio}, and artificial intelligence~\cite{LUUKKONEN2023102537}, where also many studies involve four or more objectives. The most prominent EMOA for optimizing bi-objective problems is \nsga~\cite{Deb2002} (see~\cite{VIJAI2025100550} for empirical results, or~\cite{ZhengLuiDoerrAAAI22,Dang2023} for rigorous ones), however, which is not well suited for solving problems where the number of objective increases~\cite{CAMPOSCIRO20161272,11015924}. The reason is, that in the bi-objective case, an ordering with respect to the first objective of non-dominated solutions implies also an ordering with respect to the second, which makes the crowding distance, the second tie breaker of NSGA-II, effective. However, this relation breaks down already for three objectives, which indicates the inefficiency of NSGA-II in the many-objective setting. However, \nsgaIII, a refinement of NSGA-II, uses reference points instead of crowding distance to ensure that the solution set will be well-distributed across the objective space in a very natural way (see ~\cite{DebJain2014}). It has been shown both theoretically ~\cite{WiethegerD23,OprisNSGAIII} and empirically ~\cite{CAMPOSCIRO20161272} that \nsgaIII effectively optimizes many-objective problems, which is a challenging task, as the size of the Pareto front and the number of incomparable solutions grow exponentially with the number of objectives. However, these first rigorous runtime analyses of the state-of-the-art \nsgaIII were published just a few years ago. As a result, its theoretical understanding still lags behind its practical achievements, although some progress has been made since then~\cite{Opris2025multimodal,OPRIS2026,Opris26PopDyn}. There is still a significant gap in our theoretical understanding of how crossover contributes to many-objective optimization, despite its importance as a fundamental operator in evolutionary computation~\cite{10.1145/3009966} which hinder our general understanding when and why MOEAs perform well. Beyond up to two papers~\cite{Opris2025,OPRIS2026}, which theoretically demonstrate that crossover yields an exponential runtime speedup on handcrafted benchmark functions, we are not aware of any further theoretical results, in contrast to the bi-objective setting~\cite{Dang2023,DoerrQu2023a}. But further empirical results on many different multi-objective constrained, and unconstrained problems~\cite{CrossoverEmpirical1}, and on multi- and many-objective Knapsack problems~\cite{CrossoverKnap} show the huge potential of crossover operators in many-objective optimization. 

\textbf{Our contribution:} In this paper, we build on the considerations from~\cite{Opris2025multimodal,DoerrQu2023a}, and provide a theoretical runtime analysis of \nsgaIII on the many-objective $\JUMP$ benchmark with and without crossover for any number of objectives $m$, and show that \nsgaIII without crossover needs $O(m^2n\ln(n/m) + \delta mn^k(1+2n/m)^{m/2}/\mu)$ generations with probability $1-o(1)-e^{-(\delta-3)m}$ to cover the whole Pareto front. Here, $\mu$ denotes the population size, $n$ the problem size, $k$ is a parameter specific to the problem at hand, and $\delta>3$ is an additional parameter. The expected number of generations is $O(m^2n\ln(n/m) + mn^k(1+2n/m)^{m/2}/\mu)$. One sees that this number is asymptotically the same for a wide range of population sizes $\mu$. This robustness stems from the observation that NSGA-III can retain many individuals with the same fitness vector due to how it associates solutions with and iterates over reference points. This aspect was not considered in~\cite{DoerrNearTight}. In particular, our analysis improves the runtime bound given there for population sizes asymptotically larger than the size of the Pareto front, and it even extends the analysis from~\cite{Opris2025multimodal} to an arbitrary number of objectives. With uniform crossover, we show that with probability $1-o(1)-e^{-(\delta-3)m}$, the number of generations until the whole Pareto front is covered is at most $O(k\delta m^{k+1}n^kc^k\mu/k!)$ for \nsgaIII, and the expected number of generations is $O(km^{k+1}n^kc^k\mu/k!)$ for a suitable constant $c>0$. 
Thus, one obtains a speedup of order $\Omega((k-1)!(1+2n/m)^{m/2}/(\mu^2c^km^k))$ in the expected runtime compared to the case without crossover. 
If $k = \Omega(n/\ln(n))$, $m=O(\log(n))$ and $\mu = O((1+2n/m)^{m/2})$, this speedup even becomes exponential. In a nutshell, our proof applies the arguments from~\cite{DoerrQ23b} for the bi-objective case sequentially to all blocks. A key difference is that we also need high-probability guarantees for finding a single Pareto-optimal solution in order to obtain all of them in reasonable time in parallel. Finally, we complement our analysis by also providing a lower runtime bound of \nsgaIII without crossover on $4$-OJZJ of $\Omega(n^k/\mu)$ generations, which is by a factor of $(k-1)!/(\mu^2d^k)$ larger than the upper runtime bound derived with crossover for a constant $d>0$. This factor is also exponential if $k = \Omega(n/\ln(n))$ and $\mu = \text{poly}(n)$. Extending this lower bound to a larger number of objectives appears considerably more difficult, since the interactions between single objectives becomes more complex, and requires a much deeper understanding of the underlying population dynamics which extends the scope of this paper. However, we expect our results to extend similarly to other MOEAs such as GSEMO, SPEA2, SMS-EMOA, and variants of PAES-25~\cite{OprisPAES}.\\
\textbf{Related work:} In single-objective optimization, the benefits of crossover are much better understood than in the many-objective setting. On pseudo-Boolean benchmark problems such as \textsc{JUMP}$_k$, where a fitness valley of size~$k$ must be crossed, it has been rigorously shown that \emph{uniform} crossover yields a speedup depending on $k$ and the crossover probability~\cite{Jansen2002,Koetzing2011,Dang2017,Opris24Jump,opris2026parentselectionmechanismselitist}. An exponential performance gap in the runtime was proven in~\cite{Jansen2005c} on a function \RRRfull, which is specifically designed for $1$-point crossover. These insights have been used to prove also advantages through crossover for combinatorial optimization problems like shortest paths~\cite{Doerr2012,DOERR201312}, solving complex data clustering problems~\cite{SuttonDataClus} or NP-hard graph problems~\cite{SuttonGraph24}, or even in more complex search spaces like permutation spaces~\cite{OprisPermutation}. \\
In multiobjective optimization, only a few variants of the global simple evolutionary multiobjective optimizer (GSEMO) with crossover have been studied~\cite{Qian2013,Qian_Bian_Feng_2020,Doerr2022}, and rigorous analyses of \nsga with crossover on classical benchmark problems~\cite{DoerrQ23b} and multi-objective variants of \RRRfull~\cite{Dang2023} have been conducted. However, these results were only restricted on bi-objective problems. The theoretical analysis of \nsgaIII only succeeded recently. Based on a rigorous analysis of GSEMO on classical benchmark functions~\cite{Laumanns2004}, in \cite{WiethegerD23} the first runtime analysis of \nsgaIII on the easy $3$-\OMMfull problem was conducted. 
This was then generalized by \cite{OprisNSGAIII,Opris2025multimodal} on more than three objectives, where also key structural insights into the working principles of NSGA-III are given. Also, other pseudo-Boolean functions, where one has also to reach the Pareto front at a first glance, have been analyzed there. Then, in~\cite{DoerrNearTight,Opris2025multimodal} $m$-OJZJ has been analyzed, but without investigating crossover. 
Similar analyses have then also been conducted on other popular MOEAs like the SPEA-2~\cite{RenSPEA2} and the SMS-EMOA~\cite{Zheng_Doerr_2024}. First theoretical results which showcase that MOEAs, particularly NSGA-III and SPEA-2, are quite robust with respect to the chosen population size can be found in~\cite{Opris26PopDyn,DoerrKS2026}. However, apart from~\cite{Opris2025,OPRIS2026}, we are not aware of any theoretical results on whether or how crossover can be beneficial in many-objective optimization. Moreover, these results are limited to Royal Road functions designed for specific crossover operators. 

\section{Preliminaries}
\textbf{Notation}: For a finite set $A$, we write $|A|$ for its cardinality. For $n \in \mathbb{N}$, define $[n] := \{1, \dots, n\}$ and $1^n$ the vector of length $n$ with only ones, while $0^n$ is the corresponding vector with only zeros. Given a bit string $x \in \{0,1\}^n$, let $\ones{x}$ and $\zeros{x}$ denote the number of ones and zeros in $x$, respectively. For $x,y \in \{0,1\}^n$ denote by $H(x,y) = \sum_{i=1}^n |x_i-y_i|$ the Hamming distance of $x$ and $y$.
We use $\ln$ to denote the natural logarithm. Let $Y$ and $Z$ be random variables taking values in $\mathbb{N}_0$. We say that $Z$ \emph{stochastically dominates} $Y$ if $\Pr(Z \le c) \le \Pr(Y \le c)$ for all $c \ge 0$. Consider an $m$-objective function $f : \{0,1\}^n \to \mathbb{N}_0^m, x  \mapsto (f_1(x), \ldots, f_m(x))$, and let $f_{\max}:=\max\{f_j(x) \mid j \in [m], x \in \{0,1\}^m\}$ the maximum possible value of an objective. For two search points $x, y \in \{0,1\}^n$, we say that $x$ \emph{weakly dominates} $y$, denoted $x \succeq y$, if $f_i(x) \ge f_i(y)$ for all $i \in [m]$. If, in addition, at least one of these inequalities is strict, then $x$ \emph{dominates} $y$, written $x \succ y$. If neither $x \succeq y$ nor $y \succeq x$ holds, the two points are called \emph{incomparable}. We say that a subset $S \subseteq \{0,1\}^n$ consists only of \emph{mutually incomparable solutions} if every pair of elements in $S$ is incomparable. A solution $x$ is called \emph{Pareto-optimal} if it is not dominated by any other search point in $\{0,1\}^n$, and we call the set $\{f(x) \mid x \in \{0,1\}^n \text{ is Pareto optimal}\}$ the \emph{Pareto front}. 
For a population $P_t$ and a fitness vector $v \in \mathbb{N}_0^m$, the \emph{cover number} $c_t(v)$ is defined as the number of individuals $x \in P_t$ with $f(x) = v$. We say that $v$ is \emph{covered} if $c_t(v) \ge 1$.\\
\begin{algorithm}[t]
	Initialise $P_0 \sim \text{Unif}( (\{0,1\}^n)^{\mu})$\\
	\For{$t:= 0$ to $\infty$}{
		Initialise $Q_t:=\emptyset$\\
		\For{$i=1$ to $\mu/2$}{
                Select $a_1,a_2$ from $P_t$ uniformly at random\\
                Select $u \sim \text{Unif}([0,1])$\\
                \If{$u \leq p_c$}{
                    Create $y_1$ by uniform crossover on $a_1,a_2$\\
                    Create $y_2$ by uniform crossover on $a_1,a_2$}
                \Else{Create $y_1,y_2$ as copies from $a_1,a_2$}
			Create $z_1,z_2$ by standard bit mutation on $y_1,y_2$\\
            Update $Q_t:=Q_t \cup \{z_1,z_2\}$\\
		}
		Set $R_t := P_t \cup Q_t$\\
		Partition $R_t$ into layers $F^1_t,F^2_t,\ldots ,F^k_t$ of non-dominated fitness vectors\\
            Find $i^* \geq 1$ such that $\sum_{i=1}^{i^*-1} \lvert{F_t^i}\rvert < \mu$ and $\sum_{i=1}^{i^*} \lvert{F_t^i}\rvert \geq \mu$\\
            Compute $Y_t = \bigcup_{i=1}^{i^*-1} F_t^i$\\
            Choose $\tilde{F}_t^{i^*} \subset F_t^{i^*}$ such that $\lvert{Y_t \cup \tilde{F}_t^{i^*}}\rvert = \mu$ with Algorithm~\ref{alg:Survival-Selection}\\
		Create the next population $P_{t+1} := Y_t \cup \tilde{F}^t_{i^*}$\\
	}
	\caption{NSGA-III on an $m$-objective function $f$ with population size $\mu$ and crossover probability $p_c$}
	\label{alg:nsga-iii}
\end{algorithm}

\textbf{The \nsgaIII algorithm}: The \nsgaIII algorithm (~\cite{DebJain2014}) with crossover probability $p_c \in [0,1)$ and even population size $\mu$ is presented in Algorithm~\ref{alg:nsga-iii}. Initially, a population $P_0$ of size $\mu$ is generated by selecting $\mu$ individuals uniformly at random from $\{0,1\}^n$. In each generation $t$, an offspring population $Q_t$ of size $\mu$ is created by performing the following operations $\mu/2$ times. First, two parents $a_1$ and $a_2$ are selected uniformly at random from $P_t$. Then, \emph{uniform crossover} is applied to $(a_1,a_2)$ two times with probability $p_c$ to produce two intermediate solutions $y_1$ and $y_2$. That is, for creating one solution, and for each position $i \in [n]$ independently, the entry from $a_1$ is taken with probability $1/2$, and otherwise the entry from $a_2$. If uniform crossover is not applied, $y_1$ and $y_2$ are exact copies of $a_1$ and $a_2$. Finally, two offspring $z_1$ and $z_2$ are generated by applying \emph{standard bit mutation} to $y_1$ and $y_2$, that is, each bit is flipped independently with probability $1/n$.

\begin{algorithm}[t]
        Compute the normalisation $f^n$ of $f$\\
        Associate each $x\in Y_t \cup F_t^{i^*}$ with its reference point $\rp(x)$ such that the distance between $f^n(x)$ and the line through the origin and $\rp(x)$ is minimised\\ 
        For each $r \in \refer$, set $\rho_r:=|\{x\in Y_t \mid \mathrm{rp}(x)=r\}|$\\
        Initialise $\tilde{F}_t^{i^*}=\emptyset$ and $R':=\refer$\\
        \While{$|\tilde{F}_t^{i^*}| < \mu/2$\label{line:while}}{
        Determine $r_{\min} \in R'$ such that $\rho_{r_{\min}}$ is minimal (where ties are broken randomly)\\
        Determine $x_{r_{\min}} \in F_t^{i^*} \setminus \tilde{F}_t^{i^*}$ which is associated with $r_{\min}$ and minimises the distance between the vectors $f^n(x_{r_{\min}})$ and $r_{\min}$ (where ties are broken randomly)\label{line:association}\\
        \If{$x_{r_{\min}}$ exists}{
        $\tilde{F}_t^{i^*} = \tilde{F}_t^{i^*} \cup \{x_{r_{\min}}\}$\\
        $\rho_{r_{\min}} = \rho_{r_{\min}} + 1$\\
            \If{$\lvert{Y_t}\rvert + \lvert{\tilde{F}_t^{i^*}}\rvert = \mu$}
                {\Return{$\tilde{F}_t^{i^*}$}
                }
            }
            \Else{$R'=R' \setminus \{r_{\min}\}$}
        }
        Select $\mu - |Y_t| - \mu/2$ distinct individuals uniformly at random from $F_t^{i^*} \setminus \tilde{F}_t^{i^*}$, and add them to $\tilde{F}_t^{i^*}$\\
        \Return{$\tilde{F}_t^{i^*}$}
        
	\caption{Selection procedure utilizing a set $\refer$ of reference points for maximizing a function, including uniform selection.}
	\label{alg:Survival-Selection}
\end{algorithm}

During the survival selection, the parent and offspring populations $P_t$ and $Q_t$ are merged into $R_t$ and $R_t$ is updated by partitioning $R_t$ into layers $F^1_t,F^2_t,\dots$ using the \emph{non-dominated sorting algorithm}~\cite{Deb2002} where $F^1_t$ consists of all non-dominated individuals, and $F^i_t$ for $i>1$ of individuals only dominated by those from $F^1_t,\dots,F^{i-1}_t$. Then the critical rank $i^*$ with $\sum_{i=1}^{i^*-1} \lvert{F_t^i}\rvert < \mu$ and $\sum_{i=1}^{i^*} \lvert{F_t^i}\rvert \geq \mu$ is determined (i.e. there are fewer than $\mu$ search points in $R_t$ with a lower rank than $i^*$, but at least $\mu$ search points with rank at most $i^*$). All individuals with a lower rank than $i^*$ are included in $P_{t+1}$, while the remaining individuals are selected from $F_t^{i^*}$ using Algorithm~\ref{alg:Survival-Selection}. Hereby, a normalized objective function $f^n$ is computed and then each individual with rank at most $i^*$ is associated with reference points. For the first, we use the normalization procedure from~\cite{WiethegerD23} which can be also used for maximization problems as shown in~\cite{OprisNSGAIII}. We omit detailed explanations as they are not needed for our purposes. For an $m$-objective function $f\colon \{0,1\}^n\rightarrow \mathbb{N}_0^m$, the normalized fitness vector $f^{n}(x):=(f_1^n(x),\dots,f_m^n(x))$ of a search point $x$ is computed as
\begin{align*}
f_j^n(x)
    =\frac{f_j(x)-y_j^{\min}}{y_j^{\text{nad}}-y_j^{\min}}
\end{align*}
for each $j \in [m]$ where $y^{\text{nad}}:=(y_1^{\text{nad}}, \ldots, y_m^{\text{nad}})$ and 
$y^{\min}:=(y_1^{\min}, \dots, y_m^{\min})$ from the objective space are called
\emph{nadir} and \emph{ideal} points, respectively. Computing the nadir point is not trivial and we have $y_j^{\text{nad}} \geq \varepsilon_{\text{nad}}$, and $y_j^{\text{min}} \leq y_j^{\text{nad}} \leq y_j^{\text{max}}$ for every $j \in [m]$ where $\varepsilon_{\text{nad}}$ is a positive threshold set by the user (see~\cite{Blank2019} or~\cite{WiethegerD23} for the details). Further, $y_j^{\max}$ and $y_j^{\min}$ are the maximum and minimum value in objective $j$ from all search points seen so far (i.e. from $P_0,Q_0,\ldots , P_t,Q_t$). 
After computing the normalisation, each individual $x$ is associated with the reference point $\text{rp}(x)$ such that the distance between $f^n(x)$ and the line through the origin and $\text{rp}(x)$ is minimal. 
We use the same set of reference points $\refer$ as proposed in~\cite{DebJain2014}. The points are defined as
\[
\left\{\left(\frac{a_1}{p}, \ldots ,\frac{a_m}{p} \right) 
    \text{ } \Big| \text{ }  
    (a_1,\dots,a_m) \in \mathbb{N}_0^m, 
    \sum_{i=1}^m a_i = p
\right\}
\]
where $p \in \mathbb{N}$ is a parameter one can choose according to the fitness function $f$. These are uniformly distributed on the simplex determined by the unit vectors $(1,0,\dots,0)^{\intercal},(0,1,\dots,0)^{\intercal},\dots,(0,0,\dots,1)^{\intercal}$. 

Then, if the number of all individuals already chosen for the critical layer is at most $\mu/2$, one iterates through all the reference points where the reference point with the fewest associated individuals that are already selected for the next generation $P_{t+1}$ is chosen. A reference point is omitted if it only has associated individuals that are already selected for $P_{t+1}$ and ties are broken uniformly at random. Next, from the individuals associated to that reference point who have not yet been selected, the one closest to the chosen reference point is selected for the next generation, where ties are again broken uniformly at random. Once the required number of individuals is reached, or $|F_t^{i^*}| = \mu/2$, the selection ends. If after this selection procedure still $\lvert{Y_t}\rvert + \lvert{\tilde{F}_t^{i^*}}\rvert < \mu$, then the remaining $\mu-|Y_t|-|\tilde{F}_t^{i^*}| = \mu-|Y_t|-\mu/2 = \mu/2-|Y_t|$ individuals from $F_{i^*} \setminus \tilde{F}_t^{i^*}$ are chosen uniformly at random. Such a uniform selection strategy has been shown to be successful, since it enables movement on fitness neutral environments (called \emph{plateaus}), which helps to build up and preserve population diversity~\cite{Opris24Jump,Dang2017,DoerrQu2023a}. 
The following result can be formulated and proven as in \cite{Opris2025multimodal}. For completeness, we provide a proof in the appendix.

\begin{restatable}{lemma}{general}
\label{lem:NSGAIII-General}
Consider \nsgaIII on an $m$-objective function $f$ with Pareto front $F$, and assume that $\varepsilon_{\mathrm{nad}} \geq f_{\max}$. Let $\mathcal{R}_p$ denote a set of reference points, with $p \geq 2 m^{3/2} f_{\max}$. Let $P_t$ be the population at iteration $t$. Let $S$ be a maximum set of mutually incomparable solutions, 
let $v \in F$, and $0 \leq \alpha \leq \lfloor{\mu/(2|S|)}\rfloor$. Then if $c_t(v) \geq \alpha$ then also $c_{t+1}(v) \geq \alpha$. 
\end{restatable}


\textbf{The many-objective $\JUMP$ benchmark:} This benchmark has been defined the first time in~\cite{Zheng_Doerr_2024} and is defined as follows, where $\JUMP:=(f_1, \ldots , f_m)$, $m$ is even, and $n$ is divisible by $m/2$ (see also~\cite{Qu2022PPSN} for the bi-objective version). For $2 \leq k \leq 2n/m$ the $\JUMP=\bigl(f_1(x), \ldots , f_m(x)\bigr)$ is defined as
\begin{align*}
	f_j(x)=\begin{cases}
		k+\ones{x^{\frac{j+1}{2}}},&\text{if $\ones{x^{\frac{j+1}{2}}} \leq \frac{2n}{m}-k$ or $x^j=1^{\frac{2n}{m}}$,} \\
		\frac{2n}{m}-\ones{x^j},&\text{else,}
	\end{cases}
\end{align*}
if $j \in [1,\ldots , m]$ is odd, and
\begin{align*}
	f_j(x)=\begin{cases}
        k+\zeros{x^{\frac{j}{2}}},&\text{if $\zeros{x^{\frac{j}{2}}} \leq \frac{2n}{m}-k$ or $x^{\frac{j}{2}}=0^{2n/m}$,} \\
		\frac{2n}{m}-\zeros{x^{\frac{j}{2}}},&\text{else,}
	\end{cases}
\end{align*}
if $j \in [1,\ldots , m]$ is even. We often call $k$ \emph{gap size}. For every objective there are $2n/m+1$ different values and $f_{\max} = k+2n/m$. The Pareto front~$F$ of $m$-$\OJZJ_k$ is $\{(\ell_1,2k+2n/m-\ell_1, \ldots , \ell_{m/2},2k+2n/m-\ell_{m/2}) \mid \ell_1, \ldots , \ell_{m/2} \in \{k,2k,2k+1, \ldots ,2n/m-1,2n/m,2n/m+k\}\}$, has cardinality $(2n/m-2k+3)^{m/2}$ for $k \leq n/m$, and a maximum set $S$ of mutually incomparable solutions satisfies $|F| \leq |S| \leq (2n/m+1)^{m/2}$ (see~\cite{Zheng_Doerr_2024} for proofs). 

For $v \in F$ we introduce the following notation: Denote by $r_v \in \{-1,0,1\}^{m/2}$ the $(m/2)$-dimensional vector with $(r_v)_j = -1$ if $v_{2j-1} = k$, $(r_v)_j = 1$ if $v_{2j-1} = 2n/m+k$, and $(r_v)_j = 0$ if $2k \leq v_{2j-1} \leq 2n/m$. All search points $x$ with $f(x) \in \{v \in \mathbb{N}_0^m \mid r_v=s\}$ for $s \in \{-1,0,1\}^{m/2}$ satisfy $x^j = 0^{m/2}$ if $(r_v)_j=-1$, $x^j = 1^{m/2}$ if $(r_v)_j=1$, and $k \leq \ones{x^j} \leq 2n/m-k$ if $(r_v)_j=0$ for all $j \in [m/2]$. 



\section{An Upper Bound Without Crossover}

To compare the performance of \nsgaIII with and without crossover, we first generalize a result from \cite{Opris2025multimodal} for \nsgaIII on $\JUMP$ to an arbitrary number of objectives. This generalization requires a much more refined proof, since for large $m$ the number of Pareto-optimal points can be exponentially. Interestingly, the runtime bound from \cite{Opris2025multimodal} carries over directly to this setting in terms of generations. In addition, we provide an upper bound on the runtime that holds with high probability. Our results are formulated for arbitrary crossover probabilities $p_c \in [0,1)$, as parts of the analysis are later reused for the crossover case. However, the analysis here relies only on mutation steps.


\begin{theorem}
\label{thm:whole-covering}
    Consider \nsgaIII on $f:=\JUMP$ for $2 \leq k \leq n/(2m)+1$, $(1+2n/m)^{m/2} \leq \mu/2 \leq n^{o(n)}$, crossover probability $p_c \in [0,1)$, and a number $m$ of objectives with $2 \leq m \leq n/2$. Further, assume the same conditions as in Lemma~\ref{lem:NSGAIII-General}, and let $\delta > 0$. Then 
    with probability at least 
    $1-e^{-(\delta - 3)m} -O(1/n^4)$ the number of generations until the whole Pareto front $F$ is covered is at most 
    $$O\left(\frac{m^2 n \ln (n/m)}{1-p_c} + \frac{\delta m n^k(1+2n/m)^{m/2}}{\mu(1-p_c)}\right).$$ The expected number of generations is $$O\left(\frac{m^2 n \ln (n/m)}{1-p_c} + \frac{m n^k(1+2n/m)^{m/2}}{\mu(1-p_c)} \right).$$
\end{theorem}

\begin{proof}
At first we prove that, with probability at most $e^{-\Omega(n^2)}$, there exists no Pareto optimal individual $x$ with $k \le \ones{x_j} \le 2n/m-k$ for all $j \in [m/2]$ after initialization. Then, with probability $1-e^{-\Omega(n^2)}$, a single generation suffices to create such a Pareto optimal individual since any individual $x$ can be created with probability at least $n^{-n}$, independently of whether crossover is executed or not. This implies that the expected number of generations for creating a Pareto optimal individual is at most $1+n^n \cdot e^{-\Omega(n^2)} = 1 + o(1)$.

\begin{restatable}{lemma}{initialization}
\label{lem:init}
With probability at most $e^{-\Omega(n^2)}$, there exists no Pareto optimal individual $x$ with $k \le \ones{x_j} \le 2n/m-k$ for all $j \in [m/2]$ after initialization.
\end{restatable}

After a successful initialization, we explore search points $x$ satisfying $r_{f(x)} = s$ for all possible $s \in \{-1,0,1\}^{m/2}$ in parallel. Specifically, we determine a suitable number of generations such that, for any fixed $s$, a search point $x$ with $r_{f(x)} = s$ is found with high probability. We then apply a union bound over all $3^{m/2}$ possible $s$, which shows that, still with high probability, a corresponding $x$ with $r_{f(x)} = s$ is found for every $s$ within that time. \\
We fix an $s \in \{-1,0,1\}^{m/2}$ and
define \emph{Step~$j$} as the process that starts when there is a point $y \in P_t$ with $(r_{f(y)})_i = s_i$ for all $i < j$ and $(r_{f(y)})_i = 0$ for $i \geq j$, and ends when a search point $z$ is generated with $(r_{f(z)})_i = (r_{f(y)})_i$ for all $i \leq j$, and $(r_{f(z)})_i = 0$ for $i>j$. Note that Step~$1$ directly builds on a successful initialization (since then there exists a search point $z \in P_t$ with $r_{f(z)} = 0^{m/2}$), and that the desired $z$ is created after Step~$m/2$. For Step~$j$ we fix such a $y \in P_t$. Then in all future iterations there is always a $y' \in P_t$ with $f(y')=f(y)$ by Lemma~\ref{lem:NSGAIII-General}, and hence, we can never fall back in steps. Further, we define $B_j:=\{v \in F \mid v_i=f_i(y) \text{ for every }i \in [m] \setminus \{2j-1,2j\} \text{ and } v_{2j-1} \in \{2k, \ldots , n-2k\}\}$. 
Note that $f(y) \in B_j$. Now we estimate the number of generations to finish Step~$j$.

\begin{itemize}
	\item At first we estimate the number $X_j$ of generations until the complete set $B_j$ is covered. Particularly, that for every $v \in B_j$, there is $x^* \in P_t$ with $f(x^*)=v$.
	\item After $X_j$ generations, we estimate the number of generations $Y_j$ until every Pareto optimal fitness vector $v \in B_j$ has a cover number of at least $\lfloor{\mu/(2(1+2n/m)^{m/2})}\rfloor$.
	\item Finally, after $X_j + Y_j$ generations, estimate the number of generations $Z_j$ until a desired $z$ is created by possibly crossing the fitness valley of size $k$ if necessary. 
\end{itemize}

Then the total number of generations until a desired $x$ is created is stochastically dominated by $\sum_{j=1}^{m/2} (X_j + Y_j + Z_j)$. We now derive tail bounds for the random variables $X:=\sum_{j=1}^{m/2} X_j$, $Y:=\sum_{j=1}^{m/2} Y_j$, and $Z:=\sum_{j=1}^{m/2} Z_j$ separately in three consecutive lemmas. We only provide some proof sketches due to space restrictions and similar ideas to the considerations from~\cite{Opris2025multimodal}. Their full proofs can be found in the appendix. 
  
\begin{restatable}{lemma}{tailboundX}
\label{lem:tailboundX}
Fix $\beta>0$. Then for $j \in [m/2]$ and $n$ sufficiently large,
\begin{align*}
   \Pr \left(X_j \geq \frac{4+16m}{1-p_c} \cdot n \cdot \ln\Bigl(\frac{2n}{m}\Bigr)\right) 
    \leq \left(\frac{2n}{m}\right)^{-4m}.
    \end{align*}
    Particulary,
    \begin{align*}
    \Pr \left(X \geq \frac{4+16m}{1-p_c} \cdot \frac{m n}{2} \cdot \ln\Bigl(\frac{2n}{m}\Bigr)\right) \leq \frac{m}{2} \cdot \left(\frac{2n}{m}\right)^{-4m}.
    \end{align*}
\end{restatable}

We fix an uncovered Pareto-front fitness vector $v$ and bound the number of generations until it is covered with high probability. A union bound over all $v \in B_j$ then yields a high-probability bound for covering all such vectors, which extends to $X$ via another union bound by multiplying with $m/2$. We next estimate $Y$.

\begin{restatable}{lemma}{TailboundY}
  \label{lem:tailboundY}
   For $j \in [m/2]$ we have
   \begin{align*}
   \Pr\left(Y_j \geq \frac{3207n}{1-p_c}\right) 
   \leq 2n^2e^{-4n}.
   \end{align*}
   Particulary,
   \begin{align*}
   \Pr\left(Y \geq \frac{3207nm}{2(1-p_c)}\right) \leq mn^2e^{-4n}.
   \end{align*}
\end{restatable}

The idea behind proving that each fitness vector from $B_j$ attains the desired cover number after $\Omega(n/(1-p_c))$ generations is again based on a parallelization argument. We split the proof into two phases. In the first phase, we determine an expected number of generations such that the cover number of a fixed $v \in B_j$ is at least $\Omega(n)$ with high probability, which follows from repeatedly cloning an individual $x$ with $f(x)=v$. In the second phase, we show that the desired cover number $c_t(v)$ is reached by applying a classical Chernoff bound to the number of newly created individuals in a single generation, which is $\Omega(n)$ in expectation. By a union bound, we obtain the claim for all $v \in B_j$, and finally also for $Y$. Next, we extimate $Z$.
    
\begin{restatable}{lemma}{TailboundZ}
\label{lem:tailboundZ}
For $\gamma:=\lfloor{\mu/(2(1+2n/m)^{m/2})}\rfloor$ and $j \in [m/2]$ the random variable $Z_j$ is stochastically dominated by a geometrically distributed random variable $Z_j'$ with success probability $\sigma_k:=\frac{(1-p_c)\gamma/(2en^k)}{1+(1-p_c)\gamma/(2en^k)}$. Additionally, for $\delta>1$
\begin{align*}
\Pr\left(Z \geq \frac{m(1+8\delta)}{2}\Bigl(1+\frac{2 e n^k}{(1-p_c) \gamma}\Bigr)\right) \leq e^{-\delta m}.
\end{align*}
\end{restatable} 

We estimate $Z_j$ pessimistically as the number of generations required to select an individual with $2n/m-k$ ones (or, symmetrically, $2n/m-k$ zeros) in block $j$, of which there are at least $\lfloor{\mu/(2(1+2n/m)^{m/2})}\rfloor$ many, and then flip $k$ specific bits in block $j$ to cross the fitness valley. Again, we apply a union bound to estimate $Z$. 

Finally, we are in a position to apply a union bound to the random variables $X$, $Y$ and $Z$. Using Lemmas~\ref{lem:tailboundX},~\ref{lem:tailboundY}, and~\ref{lem:tailboundZ}, respectively, we obtain for every $\delta \geq 1$ and 
\begin{align*}
K(k,m,&n,p_c,\delta):= \frac{4+16m}{1-p_c} \cdot \frac{mn}{2} \cdot \ln\Bigl(\frac{2n}{m}\Bigr) + \frac{3207nm}{2(1-p_c)}\\
&+ \frac{m(1+8\delta)}{2}\Bigl(1+\frac{2 e n^k}{(1-p_c) \gamma}\Bigr) \\
&= O\left(\frac{m^2 n \ln (n/m)}{1-p_c} + \frac{\delta m n^k(1+2n/m)^{m/2}}{(1-p_c) \mu}\right)
\end{align*}
the inequality 
\begin{align*}
&\Pr(X+Y+Z \geq K(k,m,n,p_c,\delta)) \\
&\leq m/2 \cdot (2n/m)^{-4m} + mn^2e^{-4n} + e^{-\delta m} =:p(m,n,\delta).
\end{align*}
So after $K(k,m,n,p_c,\delta)$ generations, a desired $x$ with $r_{f(x)} = s$ is created with probability at most $p(m,n,\delta)$, after a successful initialization. A union bound on at most $3^{m/2}$ such possible $x$ with different $r_{f(x)}$-values shows that with probability at most (due to $2n/m \geq 4$)
$$3^{\frac{m}{2}}p(m,n,\delta) \leq m(2n/m)^{-3m}/2 + mn^2e^{-3n} + e^{-(\delta-3)m}$$
that for each $s \in \{-1,0,1\}^{m/2}$ there is $x$ with $r_{f(x)} = s$. Suppose that this happens. Then we cover the Pareto front. Fix an uncovered $v \in F$. In an analogous way to the argument above, the number of generations required to cover $v$ is stochastically dominated by $X$. This is because there already exists a search point $y \in P_t$ with $r_{f(y)} = r_v$ and hence, for every block $j \in [m/2]$ with $(r_v)_j = 0$, the time to create a search point $z$ satisfying $(f_{2j-1}(z), f_{2j}(z)) = (v_{2j-1}, v_{2j})$, while not changing all remaining blocks, is stochastically dominated by $X_j$. By Lemma~\ref{lem:tailboundX}, we obtain 
$$\Pr\left(X \geq \frac{(4+16m) mn \ln(2n/m)}{2(1-p_c)}\right) \leq \left(\frac{2n}{m}\right)^{-4m}.$$
By a union bound over all such possible search points, we obtain that the whole Pareto front is covered in an additional amount of $(4+16m) \cdot mn/2 \cdot  \ln(2n/m)/(1-p_c)$ generations with probability at most $(1+2n/m)^{m/2} \cdot (2n/m)^{-4m} \leq (2n/m)^{-2m}$. Hence, for $p:=p(m,n,\delta)$, the entire Pareto front is covered within
\begin{align*}
K&(k,m,n,p_c,\delta) + \frac{4+16m}{1-p_c} \cdot \frac{mn}{2} \cdot \ln\Bigl(\frac{2n}{m}\Bigr) =\\ O&\left(\frac{m^2 n \ln (n/m)}{1-p_c} + \frac{\delta m n^k(1+2n/m)^{m/2}}{(1-p_c) \mu} \right)
\end{align*}
generations with probability at least
\begin{align*}
1-3^{m/2}p - (2n/m)^{-2m} = 1 - e^{-(\delta-3)m} - O(1/n^4)
\end{align*}
where we also used that the functions $h_1: \text{}]0,n/2[\text{} \to \mathbb{R}, x \mapsto (2n/x)^{-2x},$ and $h_2:\text{}]0,n/2[ \text{} \to \mathbb{R}, x \mapsto x/2 \cdot (2n/x)^{-2x},$ are strictly monotone decreasing (see Lemma~\ref{lem:function-monotone} in the appendix). The latter propability also includes the event that in the first generation a successful initialization happens.

Now it remains to estimate the expected time until the whole Pareto front is covered. Here we assume $\delta=4$. If, after a successful initialization, the Pareto front is not covered after $O(m^2n\ln(n/m)/(1-p_c) + m n^k(1+2n/m)^{m/2}/((1-p_c)\mu))$ generations (which happens with probability at most $e^{-m} + O(1/n^4)$), then we repeat all the arguments from above. The expected number of such periods is $1/(1-e^{-m} - O(1/n^4)) = O(1)$, concluding the proof of Theorem~\ref{thm:whole-covering}.
\end{proof}


\section{An Upper Bound with Crossover}

Now we show that \nsgaIII with crossover can be much more efficient when optimizing $\JUMP$ than without, especially for large $k$. The reason is that a pair $x,y$ with $\ones{x^j}=\ones{y^j} = 2n/m-k$ and maximum Hamming distance $H(x^j,y^j) = 2k$ can be created via mutation within $k$ generations, and then a recombination of those $x^j$ and $y^j$ leads to the all one string in block $j$ with probability $\Omega(1/4^k)$. This probability is by a factor $n^k/4^k$ larger than the probability to cross the fitness valley in block $j$ via mutation.

\begin{theorem}
\label{thm:whole-covering-crossover}

Consider \nsgaIII on $f:=\JUMP$ for $2 \leq k \leq n/(2m)+1$, $(1+2n/m)^{m/2} \leq \mu/2 \leq n^{o(n)}$, crossover probability $p_c \in (0,1)$, and a number $m$ of objectives with $2 \leq m \leq n/2$. Further, assume the same conditions as in Lemma~\ref{lem:NSGAIII-General}, and let $\delta > 0$. Then, 
with probability at least 
$1-e^{-(\delta - 3)m} -O(1/n^4)$, the number of generations until the whole Pareto front $F$ is covered is at most 
$$O(k\delta m^{k+1}n^k (80e)^k \mu/(p_c k!(1-p_c)^k)).$$ 
The expected number of generations is $$O(k m^{k+1}n^k (80e)^k \mu/(p_c k!(1-p_c)^k)).$$ 
\end{theorem}

\begin{proof}
We closely follow the ideas used in the proof of Theorem~\ref{thm:whole-covering}. For the initialization, with probability $1-e^{-\Omega(n^2)}$, a single generation suffices to initialize an individual $x$ with $k \leq \ones{x^j} \leq 2n/m-k$ for all $j \in [m/2]$. The expected number of generations is also $1 + o(1)$. Now we define Step~$j$ and the corresponding set $B_j$ as in the proof of Theorem~1. Let $X_j$ denote the number of generations until, for every $v \in B_j$, there exists a search point $x^* \in P_t$ with $f(x^*) = v$. After these $X_j$ generations, let $Z_j^{\text{cross}}$ be the number of additional generations until a desired $z$ is created. As in Lemma~\ref{lem:tailboundX}, we see that 
\begin{align*}
\Pr \left(X \geq \frac{4 + 16m}{1-p_c} \cdot \frac{mn}{2} \cdot \ln\Bigl(\frac{2n}{m}\Bigr)\right) \leq \frac{m}{2} \left(\frac{2n}{m}\right)^{-4m}.
\end{align*}

Now we estimate $Z$. To estimate $Z_j$, we consider $k+1$ consecutive generations. In the first $k$ generations, we create two suitable individuals $y_1,x_2$ with maximum Hamming distance in block $j$ which can then be used for recombination in the $(k+1)$-th generation. 

\begin{lemma}
\label{lem:tailboundZcross}
For $j \in [m/2]$ the random variable $Z_j$ is 
stochastically dominated by $(k+1)Z_j^\text{cross}$ where $Z_j^\text{cross}$ is a geometrically distributed random variable with success probability $\sigma_k:=\frac{1}{10^k} \frac{p_c/(2 \cdot 4^{k+1}\mu)}{1+p_c/(2 \cdot 4^{k+1}\mu)}\prod_{\ell=1}^k \frac{(1-p_c) \ell/(2emn)}{1+(1-p_c) \cdot \ell/(2emn)}$. Additionally, for $\delta>1$,
\begin{align*}
\Pr \left(Z \geq (k+1)(1+8\delta)m/(2\sigma_k) \right) \leq e^{-\delta m}.
\end{align*}
\end{lemma} 

\begin{proof}
As in the proof of Lemma~\ref{lem:tailboundX}, for the desired search point $z$ to create, we may assume that $z_j=1^{m/2}$. Let $s:=r_{f(z)}$. Consider a sequence of $k+1$ generations as follows: Within the first $k$ generations, one may create a suitable pair of individuals $x_1, x_2$ such that $\ones{x_1^j} = \ones{x_2^j} = 2n/m - k$ and the Hamming distance in block $j$ is maximized, particularly $H(x_1^j, x_2^j) = 2k$, while the individuals coincide in all other blocks. Moreover, for all $i < j$, we have $(r_{f(x_1)})_i = (r_{f(x_2)})_i = s_i$ and we have $(r_{f(x_1)})_i = (r_{f(x_2)})_i = 0$ for $i>j$. In the $(k+1)$-th generation, perform uniform crossover on $x_1$ and $x_2$, and do not flip any bit during mutation, to create a suitable search point $z$ with $z^j = 1^{2n/m}$, particularly $(r_{f(z)})_i = s_i$ for all $i \le j$, and $(r_{f(z)})_i = 0$ for all $i > j$. So for generation $\ell \leq k$, assume that there are two individuals $x_{\ell,1},x_{\ell,2}$ with $\ones{x_{\ell,1}^j} = \ones{x_{\ell,2}^j} = 2n/m-k$, $H(x_{\ell,1}^j,x_{\ell,2}^j) = 2(\ell-1)$, where for $\ell=1$, the individuals $x_{\ell,1},x_{\ell,2}$ are the same. They also coincide in all the other blocks, where $(r_{f(x_{\ell,1})})_i = s_i$ for $i < j$, and $(r_{f(x_{\ell,1})})_i = 0$ for $i > j$ is satisfied. Then in one trial, one may choose $x_1$ as a parent (prob. at least $1/\mu$), omit the crossover, and flip exactly two bits, namely, one from $1$ to $0$ and one from $0$ to $1$ in block $j$, both of which do not contribute to $H(x_1,x_2)$. 
Note that there are $k-(\ell-1)$ such zeros, and $2n/m-2(\ell-1)$ such ones. Hence, this happens with probability at least $(1-p_c)(k-\ell+1)/n \cdot (2n/m-(\ell-1)-k)/n \cdot (1-1/n)^{n-2} \geq (1-p_c)(k-\ell+1)/n \cdot (2n/m-2k+1)/(en) \geq (1-p_c)(k-\ell+1)/n \cdot (n/m-1)/(en) \geq (1-p_c)(k-\ell+1)/(2emn)=:q_\ell$, and in one generation with probability at least 
$$1-\left(1-\frac{q_\ell}{\mu}\right)^{\mu/2} \geq \frac{q_\ell/2}{1+q_\ell/2}.$$
These $x_1$ and $x_2$ then both survive with probability at least $1/10$: If $|F_t^1|\leq \mu$ where $|F_t^1|$ is the number of individuals with rank one, then both survive with probability $1$. Otherwise, $\mu < |F_t^1|\leq 2\mu$, and they survive with probability at least (when both are not selected within the first $\mu/2$ iterations of the reference point scheme in Algorithm~\ref{alg:Survival-Selection}) $\binom{\mu/2}{2} / \binom{|F_t^1|-\mu/2}{2} \geq \binom{\mu/2}{2} / \binom{3\mu/2}{2} \geq 1/10$ for $n$ sufficiently large.
Hence, in a sequence of $k$ generations, the probability that there is such a pair $x_1,x_2$ of individuals is at least $1/10^k \prod_{\ell=1}^k (q_\ell/2)/(1+q_\ell/2)$. For the subsequent $(k+1)$-th generation, there are two such individuals $x_1,x_2$. 
Now, choose $x_1,x_2$ as parents (prob. at least $1/\mu^2$), perform uniform crossover to create $y'$ with $(r_{f(y')})_i=s_i$ for all $i \leq j$, particularly $(y')_i = 1^{2n/m}$ and omit mutation (prob. $p_c/4^k \cdot (1-1/n)^n \geq p_c/4^{k+1}$). Hence, such a desired $y'$ is generated in this  $(k+1)$-th generation with probability at least    
$$1-\left(1-\frac{p_c}{ 4^{k+1}\mu^2}\right)^{\mu/2} \geq \frac{p_c/(2 \cdot 4^{k+1} \mu)}{1+p_c/(2 \cdot 4^{k+1} \mu)}=:p^{\text{cross}},$$
which survives by Lemma~\ref{lem:NSGAIII-General}. Hence, the probability is at least $\sigma_k=\frac{p^{\text{cross}}}{10^k} \cdot \prod_{\ell=1}^k \frac{q_\ell/2}{1+q_\ell/2}$
to create the desired $y'$ within $k+1$ generations. Hence, the random variable $Z=\sum_{j=1}^{m/2} Z_j$, is indeed stochastically dominated by $(k+1)Z^*$, where $Z^*=\sum_{j=1}^{m/2} Z_j^*$ is an independent sum of geometrically distributed random variables $Z_j^*$ with success probability $\sigma_k$. Its expectation is $m/(2\sigma_k)$. Hence, we have that 
\begin{align*}
\Pr(Z &\geq (k+1)\text{E}[Z^*] + (k+1)\lambda)
= \Pr(Z^* \geq \text{E}[Z^*] + \lambda) \\
&\leq \exp\left(-\frac{1}{4} \min \left\{\frac{\lambda^2}{d}, \lambda \sigma_k\right\}\right),
\end{align*}
for all $\lambda \geq 0$ where $d:=m/(2\sigma_k^2)$. For $\lambda = 4m \delta/\sigma_k$, where $\delta \geq 1$ is arbitrary, we observe that $\Pr(Z \geq (k+1) \cdot (1+8\delta) m/(2\sigma_k)) \leq e^{-\delta m}$ proving the lemma.
\end{proof}

Now we take a union bound to obtain with Lemmas~\ref{lem:tailboundX} and~\ref{lem:tailboundZcross} for $W:=X+Z$
\begin{align*}
\Pr&\left(W \geq \frac{(4+16m)mn \ln(2n/m)}{2(1-p_c)} + \frac{(k+1)(1+8\delta)m}{2\sigma_k} \right)\\
&\leq (2n/m)^{-4m} + e^{-\delta m}.
\end{align*}
Hence, after $O(nm \ln(n/m)/(1-p_c)+k\delta m 40^k \mu/p_c \cdot (2emn)^k/(k!(1-p_c)^k)) = O(k\delta m^{k+1}n^k (80e)^k \mu/(p_c k!(1-p_c)^k))$ generations, we see that a desired $x$ with $r_{f(x)} = s$ is created with probability at most $(2n/m)^{-4m} + e^{-\delta m}$. A union bound over at most $3^{m/2}$ choices of $x$ with different $r_{f(x)}$ values shows that, with probability at most $e^{-(3-\delta)m} + O(1/n^4)$, after this number of generations there exists for each $s \in {-1,0,1}$ an $x$ with $r_{f(x)} = s$. Conditioned on this event, the final covering step proceeds as in Theorem~\ref{thm:whole-covering} and succeeds with probability at least $1 - O(1/n^4)$ within an additional $O(m^2 n \ln(n/m)/(1-p_c))$ generations. Hence, with probability at least $1-e^{-(\delta-3)m} - O(1/n^4)$, the whole Pareto front is covered in $O(k\delta m^{k+1}n^k (80e)^k \mu/(p_c k!(1-p_c)^k))$ generations. If this does not happen, we repeat the above arguments to obtain an expected number of $O(k m^{k+1}n^k (80e)^k \mu/(p_c k!(1-p_c)^k))$ generations, where we use $\delta=4$, proving the theorem. 
\end{proof}

\section{A Lower Bound for $4$-OJZJ}

To complement the upper bound for the crossover case, we also establish a lower bound for mutation only in the case $m=4$ and for population sizes asymptotically smaller than $n^k$ by a factor of $n^2$.

\begin{theorem}
\label{lem:lower-runtime-bound}
Consider \nsgaIII on $f:=\JUMPFOUR$ for $2 \leq k \leq n/8+1$, $\mu \in O(n^{k-2})$, where crossover is turned off. 
Then the expected number of generations until the whole Pareto front is covered is at least $\Omega(n^k/\mu)$.
\end{theorem}

\begin{proof}
By a classical Chernoff bound, with probability $1-e^{-\Omega(n)}=1-o(1)$, each individual $x$ satisfies $\ones{x^j} \in \{k,\ldots , n/2-k\}$ for both $j \in \{1,2\}$ after initialization. Suppose that this happens. Then in all future populations $P_1,P_2, \ldots$ there are only individuals $x$ with $\ones{x^j} \in \{0,k, \ldots , n/2-k,n/2\}$ for at least one $j \in \{1,2\}$, since other individuals are dominated by every $y$ with $\ones{y^j} \in \{k \ldots , n/2-k\}$ for both $j \in \{1,2\}$. 

\begin{restatable}{lemma}{notjumping}
\label{lem:not-jumping}
With probability at least $1/81-o(1)$, we have that $C_t:=\{s \in \{-1,0,1\}^{2} \mid \text{ there is }x \in P_t \text{ with } r_{f(x)}=s\} \neq \{-1,0,1\}^2$ after $cn \ln(n)$ generations for a constant $c>0$, and $C_t$ can only increase by flipping either $k$ specific ones or zeros in a block $j \in \{1,2\}$.
\end{restatable}

Assume the event from Lemma~\ref{lem:not-jumping} occurs. Then flipping $k$ specified ones or zeros within a block happens with probability at most $2n^{-k}$ per trial, so the expected number of trials is at least $(1-o(1))n^k/2 = \Omega(n^k)$. This yields an expected number of $\Omega(n^k/\mu)$ generations.
\end{proof}

\section{Conclusion}
In this paper, we derived upper runtime bounds for the widely used \nsgaIII algorithm on $\JUMP$ with an arbitrary number $m$ of objectives, both with and without crossover, in expectation and with high probability. Notably, crossover can significantly speed up the runtime and even lead to exponential improvements in suitable parameter regimes, for example when the gap parameter $k$ is large while the number of objectives and the population size are small. To complement these findings, we also established a lower runtime bound for the case $m=4$. We hope these results improve the understanding of crossover in MOEAs, especially for escaping local optima and improving search efficiency. Future work could investigate lower bounds for the crossover case, and for $m>4$ if crossover is turned off. The latter is more challenging due to population dynamics, and we still lack a clear understanding of how non-Pareto-optimal individuals behave in the population. For crossover, the challenge is that knowledge about objective vectors alone are insufficient. One has to know also about the distribution of the genotypes of individuals. Another direction is to explore crossover in more complex many-objective settings, such as combinatorial optimization problems like shortest paths, or permutation spaces, like the many-objective flow-shop scheduling problem.
\newpage

\bibliographystyle{named}
\bibliography{ijcai26}

\cleardoublepage
\appendix
\onecolumn

\section*{Supplementary Material for Paper 6164: ''On the Impact of Crossover in Many-Objective Optimization: A Runtime Analysis of NSGA-III''}

This document contains the proofs that we omitted in the main paper in full details, due to space restrictions.

\general*
\begin{proof}
    The \nsgaIII iterates through all reference points, always preferring a reference point $r$ with the fewest associated individuals chosen for $P_{t+1}$ so far, as long as $\lvert{\tilde{F}_t^{i^*}}\rvert < \mu/2$ (see Line~\ref{line:while} in Algorithm~\ref{alg:Survival-Selection}), and selecting an individual for $P_{t+1}$ associated to $r$ (see Line~\ref{line:association} in Algorithm~\ref{alg:Survival-Selection}). By Lemma~3.3 in~\cite{OprisNSGAIII}, two Pareto optimal search points with distinct fitness are associated to two different reference points. Hence, \nsgaIII iterates at least $\alpha$ times through all reference points with at least $\alpha$ many associated individuals to find $P_{t+1}$. Hence, the cover number $c_{t+1}(v)$ of $v$ with respect to $P_{t+1}$ is still at least $\alpha$.
\end{proof}

\begin{lemma}
\label{lem:function-monotone}
The following properties are satisfied for $n \in \mathbb{N}$.
\begin{itemize}
    \item[(i)] Let $c>0$. Then the function $g: \text{}]0,2cn/e[\text{} \to \mathbb{R}, x \mapsto (c+2cn/x)^{x/2},$ is strictly monotone increasing.
    \item[(ii)] The function $h: \text{}]0,n/2]\text{} \to \mathbb{R}, x \mapsto (2n/x)^{2x},$ is strictly monotone increasing.
    \item[(iii)] The function $\rho:\text{}]0,n/2[ \text{} \to \mathbb{R}, x \mapsto x/2 \cdot (2n/x)^{2x},$ is strictly monotone increasing.
\end{itemize}
\end{lemma}

\begin{proof}
(i): By taking the logarithm, it is enough to prove that $\tilde{g}: \text{}]0,n/4[\text{} \to \mathbb{R}, x \mapsto x/2 \cdot \ln(c+2cn/x),$ is strictly monotone increasing. We have for its derivative 
$$\tilde{g}'(x) = \frac{1}{2} \cdot \ln \left(c+\frac{2cn}{x}\right) -  \frac{x}{2} \cdot \frac{2n/x^2}{1+2n/x} = \frac{1}{2} \cdot \ln \left(c+\frac{2cn}{x} \right) - \frac{n}{x} \cdot \frac{1}{1+2n/x} = \frac{1}{2} \cdot \ln \left(c+\frac{2cn}{x} \right) - \frac{1}{x/n+2}.$$
The latter is at least zero if and only if $\ln(c+2cn/x) \geq  2/(x/n+2)$ which holds since $\ln(c+2cn/x) \geq 1$ due to $2cn/x \geq e$.\\
(ii): We obtain $\ln(h) = 2x \cdot \ln(2n/x)$ and hence, 
$$\ln(h)'=2\ln \left(\frac{2n}{x}\right) - 2x \cdot \frac{1}{2n/x} \cdot \frac{2n}{x^2} = 2\ln \left(\frac{2n}{x} \right) - 2 \geq 0$$
since $x \leq n/2$.\\
(iii): We obtain $\ln(\rho) = \ln(x/2) + 2x \cdot \ln(2n/x)$ and hence, $$\ln(\rho)'=\frac{1}{x} + 2\ln\left(\frac{2n}{x}\right) - 2x \cdot \frac{1}{2n/x} \cdot \frac{2n}{x^2}= \frac{1}{x} + 2\ln \left(\frac{2n}{x} \right) - 2 \geq 0$$ since $x \leq n/2$.
\end{proof}

\initialization*

\begin{proof}
Suppose $m = 2$. The probability is at least $1-e^{-\Omega(n)}$ that $k \leq \ones{x} \leq n-k$. By applying a union bound on all $\mu$ individuals, the probability is at most $e^{-\Omega(\mu n)}$ that there exists no individual
$x$ with $k \le \ones{x} \le n-k$ for all $j \in [m/2]$ after initialization. Since $\mu = \Omega(n)$, this case is proven.

If $m \geq 4$, we still can estimate the probability that $k \leq \ones{x^j} \leq 2n/m-k$ for a given block $j \in [m/2]$ by $2/3$ from below. Then the probability that all individuals $x$ are not Pareto optimal after initialization is at most $(1-c^{m/2})^\mu \leq e^{-\mu (2/3)^{m/2}} \leq e^{-(2/3+4n/(3m))^{m/2}} = e^{-\Omega(n^2)}$ where the last inequality holds since the function $g: \text{}]0,4n/(3e)[\text{} \to \mathbb{R}, x \mapsto (2/3+4n/(3x))^{x/2},$ is strictly monotone increasing (by Lemma~\ref{lem:function-monotone}(i) for $c=2/3$) and $(2/3+4n/(3m))^{m/2} = \Omega(n^2)$ for $m > 4n/(3e)$. 
Further, we used $1+x \leq e^x$ for all $x \in \mathbb{R}$.
\end{proof}

\tailboundX*

\begin{proof}
  Fix an uncovered Pareto optimal fitness vector $v \in B_j$. For $\delta>0$ we estimate the probability that a solution $\hat{y}$ with $f(\hat{y})=v$ has not been created after $4(1+\delta) \cdot n/(m(1-p_c)) \cdot \ln(2n/m)$ generations. Let $L_j^{(t)}:=\{x \in P_t \mid f(x) \in B_j\}$ be the set of all search points which cover a fitness vector from $B_j$ (which is not empty by definition of Step~$j$),  and $e_j^{(t)}:=\min\{f_{2j-1}(x)-v_{2j-1}| \mid x \in L_j^{(t)}\}$ be the minimum possible distance of search points from $P_t$ in the objective space and $v$ with respect to objective $2j-1$. 
  Hence, $e_j^{(t)}$ is just the absolute difference of the number of ones occurring 
  in block $j$ of a search point $y' \in L_j^{(t)}$ with $|f_{2j-1}(y')-v_{2j-1}|=e_j^{(t)}$ and the desired $\hat{y}$. Hence, $e_j^{(t)} \leq 2n/m-2k$, and $e_j^{(t)}=0$ if $v$ is covered. 

  For $\ell \in [2n/m-2k]$ let $X'_{j,\ell}$ be the random variable defined as the number of generations $t$ with $\ell=e_j^{(t)}$. 
  Then the number of generations until there is the desired $y$ is at most $X'_j=\sum_{\ell=1}^{2n/m-2k} X'_{j,\ell}$, and
  $X_j$ is stochastically dominated by $X_j'$. To decrease $\ell$, it suffices to choose an individual $x$ with $\ell=e_j^{(t)}$ as a parent (prob. at least $1/\mu$), omit crossover (prob. $(1-p_c)$) and flip one of $\ell$ specific bits, while not changing the other ones (prob. $\ell/n \cdot (1-1/n)^{n-1} \geq \ell/(en)$). Let $\gamma_\ell:=(1-p_c)\ell/(en)$. Then, the probability for decreasing $\ell$ in one generation is at least
  \[
  1-\Bigl(1-\frac{\gamma_\ell}{\mu}\Bigl)^{\mu/2} \geq \frac{\gamma_\ell/2}{\gamma_\ell/2+1} = \frac{(1-p_c)\ell}{(1-p_c)\ell+en} \geq \frac{(1-p_c)\ell}{4n}
  \]
  where the first inequality is due to Lemma~10 in~\cite{Badkobeh2015}. Hence, for $\tilde{n}:=2n/m$ the variable $X_j'$ (and therefore also $X_j$) is stochastically dominated by the sum $\tilde{X}_j=\sum_{\ell=1}^{\tilde{n}} \tilde{X}_{j,\ell}$ of independently geometrically distributed random variables $\tilde{X}_{j,\ell}$ with success probability $(1-p_c)\ell/(4n) = (1-p_c)\ell/(2 m \tilde{n})$. With Theorem~16 in \cite{DOERR2019115} we obtain for $\delta > 0$
    \begin{align*}
    \Pr&\left(\tilde{X}_j \geq \frac{4(1+\delta) n\ln(\tilde{n})}{1-p_c}\right) = \Pr\left(\tilde{X}_j \geq \frac{2m(1+\delta) \tilde{n}\ln(\tilde{n})}{1-p_c}\right) \leq \tilde{n}^{-\delta} = e^{-\delta \ln(\tilde{n})} 
    \end{align*}
    and thus, we obtain for $X=\sum_{j=1}^{m/2} X_j$ by a union bound
    \begin{align*}
    \Pr(X &\geq 2(1+\delta)mn\ln(\tilde{n})/(1-p_c)) \\
    &= \Pr(X \geq m/2 \cdot 4(1+\delta)n\ln(\tilde{n})/(1-p_c))\\
    &\leq \frac{me^{-\delta \ln(\tilde{n})}}{2}.
    \end{align*}
    Now we plug in $\delta=4m$ to further estimate
\begin{align}
\label{eq:help}
    \frac{m e^{-4m \ln(\tilde{n})}}{2} &= \frac{m e^{-4m}}{2} \cdot  e^{-4m(\ln(\tilde{n})-1)}\\
    &\leq \frac{m \cdot e^{-4m}}{2} \cdot e^{-8\ln(n)+8} \leq \frac{me^{-4m+8}}{2n^8},
\end{align}
    where we used that the function
    $$h: \; ]0,2n/e^2[ \; \to \mathbb{R}, x \mapsto 4x (\ln(2n/x)-1),$$
    is strictly monotone increasing (due to $h'(x) = 4 (\ln(2n/x) -1) - 4x/(2n/x) \cdot 2n/x^2 = 4 \ln(2n/x) - 4 -4x^2/(2n) \cdot 2n/x^2 = 4 \ln(2n/x) - 8 \geq 0$) which implies that the first inequality in Equation~\ref{eq:help} is satisfied for $m < 2n/e^2$. For $2n/e^2 \leq m \leq n/2$ we see that $-4m(\ln(\tilde{n})-1) < -8\ln(n)+8$ for $n$ sufficiently large (since this is equivalent to $m(\ln(\tilde{n})-1) > 2\ln(n)-2$, which is satisfied if $n/e^2 \cdot (\ln(4)-1)  > \ln(n)-1$, and the latter holds for $n$ sufficiently large), concluding the proof of Lemma~\ref{lem:tailboundX}.
    \end{proof}

\TailboundY*

\begin{proof}
Consider a fitness vector $v \in B_j$. Note that $c_t(v) \geq 1$. We show that $c_t(v) \geq \gamma := \lfloor{\mu / (2(1 + 2n/m)^{m/2})}\rfloor$ after $3207 n/(1-p_c)$ generations with probability at least $1 - e^{-4n}$. By Lemma~\ref{lem:NSGAIII-General}, $c_t(v) \leq \gamma$ cannot decrease, because any maximum set of mutually incomparable solutions $S$ satisfies $|S| \leq (2n/m + 1)^{m/2}$, which implies $c_t(v) \leq \gamma \leq \lfloor{\mu / (2|S|)\rfloor} \leq \lfloor{\mu / (2(1 + 2n/m)^{m/2})\rfloor}$. We divide the run into two phases, where the second phase only applies if $\gamma > 288n$.

\textbf{Phase 1:} We have $c_t(v) \geq \nu:=\min\{\gamma,288n\}$.\\
For $j \in [\nu-1]$ let $W_j$ be a random variable that counts the number of generations $t$ with $c_t(v)=j$. Then the number of generations until the cover number of $v$ is at least $\nu$ is at most $W:=\sum_{j=1}^{\nu-1} W_j$. Note that $c_t(v)$ can be increased in one trial by choosing an individual $x'$ with $f(x')=v$ as parent (prob. $1/\mu$), omitting crossover (prob. $1-p_c$) and flipping no bits (prob. $1/\mu \cdot (1-1/n)^n \geq 1/(4 \mu)$). Hence, the probability of increasing $c_t(v)$ in one generation is at least 
\begin{align*}
1-\left(1-\frac{1-p_c}{4 \mu}\right)^{\mu/2} &\geq \frac{(1-p_c)/8}{1+(1-p_c)/8} \\
&= \frac{1-p_c}{8 + 1-p_c} \geq \frac{1-p_c}{9}
\end{align*}
where the first inequality is due to Lemma~10 in~\cite{Badkobeh2015}. Hence, $W$ is stochastically dominated by an independent sum $W':=\sum_{j=1}^{\nu-1}W_j'$ of geometrically distributed random variables $W_j'$ with success probability $(1-p_c)/9$. Note that $\expect{W} \leq \expect{W'} \leq 9\nu/(1-p_c) \leq 2592n/(1-p_c)$ and hence, by Theorem~15 in~\cite{DOERR2019115}, we obtain for $d:= 81\nu/(1-p_c)^2$, and $\lambda \geq 0$
\[
\Pr(W' \geq \expect{W'} + \lambda) \leq \exp\left(-\frac{1}{4} \min\left\{\frac{\lambda^2}{d}, \frac{\lambda(1-p_c)}{9} \right\}\right).
\]
For $\lambda=612n/(1-p_c)$ we obtain $\Pr(W \geq 3204n/(1-p_c)) = \Pr(W \geq 2592n/(1-p_c) + 612n/(1-p_c)) \leq \Pr(W' \geq 2592n/(1-p_c) + 612n/(1-p_c)) \leq \Pr(W' \geq  \expect{W'} + 612n/(1-p_c)) \leq e^{-4n}$.
    
    \textbf{Phase 2.} We have $c_t(v) \geq \gamma$.

    We can assume that $\gamma > 288n$. Let $M_t > 288n$ be the number of individuals $x$ with $f(x)=v$. Denote by $N_t$ the number of newly created individuals within a period of $\lceil{2/(1-p_c)}\rceil$ generations. Then $\text{E}[N_t] \geq M_t/4$ since $\lceil{2/(1-p_c)}\rceil$ generations consist of at least $\lceil{2/(1-p_c)}\rceil \mu/2 \geq 2\mu/(2(1-p_c)) = \mu/(1-p_c)$ trials, and in one trial such an individual is cloned with probability at least $288n(1-p_c)(1-1/n)^n/\mu \geq 288n(1-p_c)/(4\mu) = 72n(1-p_c)/\mu$ (with probability $1-p_c$ crossover is omitted, with probability at least $M_t/\mu \geq 288n/\mu$ one such individual is chosen as parent and finally no bit is flipped with probability at least $(1-1/n)^n \geq 1/4$ during mutation). By a classical Chernoff bound, $\Pr(M_{t+1} \leq M_t + M_t/6) \leq \Pr(N_t \leq 2\text{E}[N_t]/3) = \Pr(N_t \leq (1-1/3)\text{E}[N_t]) \leq e^{-\text{E}[N_t]/18} \leq e^{-M_t/72} < e^{-4n}$. Hence, $M_{t+1} \geq \min\{M_t+M_t/6,\nu\} = \min\{7M_t/6,\nu\}$ with probability at least $1-e^{-4n}$. Note that at most $n$ such periods in a row are sufficient to obtain a cover number of $v$ of at least $\gamma$ due to $(7n/6)^n \geq \mu > \gamma$ for sufficiently large $n$, and this occurs with probability at least $1-ne^{-4n}$ by a union bound on at most $n$ such periods. These periods consist of at most $3n/(1-p_c)$ generations.
    
    Combining both phases, we see by a union bound that $c_t(v) \geq \gamma$ with probability at least $1-e^{-4n}-ne^{-4n}$ after $3204n/(1-p_c) + 3n/(1-p_c) = 3207n/(1-p_c)$ generations. Note that $c_t(v) < \gamma$ cannot decrease by Lemma~\ref{lem:NSGAIII-General}(2). Hence, by a union bound on at most $2n/m-2k+1 \leq n$ vectors $v \in B_j$ , we obtain that 
    \begin{align*}
    \Pr\left(Y_j \geq \frac{3207n}{1-p_c}\right) \leq ne^{-4n} + n^2 e^{-4n} \leq 2n^2 e^{-4n}.
    \end{align*}
    The bound on $Y$ follows also by a union bound on all single steps $j \in [m/2]$.
    \end{proof}

\TailboundZ*

\begin{proof}
For the desired search point $z$ to create, we may assume that $z_j=1^{m/2}$. The other case that $z_j=0^{m/2}$ is symmetric. Further, if $k \leq \ones{z_j} \leq 2n/m-k$, there is nothing to show, since such a $z$ already exists in the population. To create the desired $z$ with $z_j=1^{m/2}$ in one trial, one can choose a $y_e$ covering a fitness vector from $B_j$ with $\ones{y_e^j} = 2n/m-k$ as a parent (prob. at least $\gamma/\mu$), omit crossover, and flip $k$ specific bits while keeping the remaining bits unchanged (prob. $(1-p_c)(1-1/n)^{n-k}/n^k \geq 1/(e n^k)$). Consequently, the probability that this happens in a single generation is at least
$$1-\Bigl(1-\frac{(1-p_c)\gamma}{e \mu n^k}\Bigl)^{\mu/2} \geq \frac{(1-p_c)\gamma/(2en^k)}{1+(1-p_c)\gamma/(2en^k)}=:p_k$$
where the first inequality is due to Lemma~10 in~\cite{Badkobeh2015}. Now, take the independent sum $Z'= Z_1'+ \ldots + Z_{m/2}'$, where $Z_j'$ is gemoetrically distributed with success probability $p_k$. Then, $Z$ is stochastically dominated by $Z'$, and we obtain 
$$\text{E}[Z'] = \frac{m}{2} \cdot \frac{1}{p_k} = \frac{m}{2} \left(1+\frac{2en^k}{(1-p_c)\gamma}\right) = \frac{m}{2}+\frac{men^k}{(1-p_c)\gamma}.$$
Hence, for $\lambda:=4\delta m+8\delta men^k/((1-p_c)\gamma) = 4\delta m/p_k = 8\delta \text{E}[Z']$ and $d:=m/(2p_k^2)$ we obtain 
\begin{align*}
\Pr(Z \geq \text{E}[Z'] + \lambda) &\leq \Pr(Z' \geq \text{E}[Z'] + \lambda)\\
&\leq \exp\left(-\frac{1}{4} \min \left\{\frac{\lambda^2}{d}, \lambda p_k\right\}\right) \\
&= \exp(-\delta m),
\end{align*}
concluding the proof of this lemma, since 
\begin{align*}
\text{E}[Z'] + \lambda &= \frac{m}{2}+\frac{men^k}{(1-p_c)\gamma} + 4\delta m+\frac{8\delta men^k}{(1-p_c)\gamma}\\
&= \frac{m(1+8\delta)}{2}\left(1+\frac{2 e n^k}{(1-p_c) \gamma}\right). \qedhere
\end{align*}
\end{proof}

\notjumping*

\begin{proof}
Note that the probability to flip $k$ specific bits in a block is $O(1/n^k)$ which happens with probability at most $O(\mu d n \ln(n)/n^k) = o(1)$ in $d n \ln(n)$ generations, since one generation consists of at most $\mu$ trials. Note also that $\mu \in O(n^{k-2})$. So we can assume that we never create $1^{n/2}$ or $0^{n/2}$ in one block by flipping $k$ specific bits. Hence, to create an individual $x$ with $\ones{x^j} \in \{0,n/2\}$ in one block $j \in \{1,2\}$, one must mutate an individual $z \in P_t$ in block $i$ that satisfies $|\ones{z^i}-\ones{x^i}| < k$ and $\ones{z^{3-i}} \in \{k,\ldots,n/2-k\}$. One also achieves $\ones{x^{3-i}} \in \{k,\ldots,n/2-k\}$ if one does not flip any bit in $z^{3-i}$. A search point $x$ with $\ones{x^j} \in \{0,n/2\}$ and $\ones{x^{3-i}} \in \{k,\ldots,n/2-k\}$ is always kept in the population. So under the condition that a search point $x$ with $\ones{x^i} \in \{0,n/2\}$ is created in one trial, the probability that $\ones{x^{3-i}} \in \{0, \ldots , k-1, n/2-k+1, \ldots , n/2\}=:L$ is at most $1-(1-1/n)^{n/2} \leq 1-1/e \leq 2/3$. If this does not happen, $x$ dominates every search point $z^*$ with $\ones{(z^*)^i}=\ones{x^i}$ and $\ones{(z^*)^{3-i}} \in L \setminus \{0,n/2\}$, and such a $z^*$ never survives. By considering both $j \in \{1,2\}$ and $i \in \{0,n/2\}$, with probability at least $(1-2/3)^4 \geq 1/81$, individuals $z^*$ with $\ones{(z^*)^j} \in \{0,n/2\}$ for a $j \in \{1,2\}$, while $\ones{(z^*)^{3-j}} \in L \setminus \{0,n/2\}$ never survive, since those are dominated by a corresponding $x$. So suppose that this happens. Then, neither $1^n$, nor $0^n$ can be created, since this requires to flip at least $k$ specific bits in one block $i \in \{1,2\}$. Now, with Lemma~\ref{lem:tailboundX}, cover all Pareto optimal vectors $v$ with $r_v \in C_1$ with probability $1-o(1)$ within $cn \ln(n)$ generations for a suitable constant $c>0$ (by at first considering $X=X_1+X_2$ since we are building on a successful initialization and then using a union bound on at most $(1+n/2)^2$ possible fitness vectors). If $C_t$ has been enlarged in this time, consider an additional number of $cn \ln(n)$ generations. Note that $C_t$ can be enlarged at most $6$ times this way (excluding $0^n$ and $1^n$), and hence, after at most $6cn \ln(n)$ generations, all fitness vectors $v$ with $r_v \in C_t$ are covered with probability at least $1-o(1)$, but still $1^n,0^n \notin P_t$. Then, we can enlarge $C_t$ again only by flipping $k$ specific bits in one block $i \in \{1,2\}$, concluding the proof, since this happens with probability $1-o(1)$.
\end{proof} 

\end{document}